\documentclass[journal]{IEEEtran}

\ifCLASSINFOpdf
\else
   \usepackage[dvips]{graphicx}
\fi
\usepackage{url}

\usepackage{graphicx}

\usepackage{amsthm}

\theoremstyle{plain}
\newtheorem{theorem}{Theorem}
\newtheorem{proposition}[theorem]{Proposition}
\newtheorem{lemma}[theorem]{Lemma}

\theoremstyle{definition}

\theoremstyle{remark}
\newtheorem{remark}[theorem]{Remark}

\newcommand{\norm}[1]{\left\lVert#1\right\rVert}

\newcommand{\bbP}{\mathbb{P}}

\newcommand{\bbR}{\mathbb{R}}

\newcommand{\calL}{\mathcal{L}}

\newcommand{\calN}{\mathcal{N}}

\newcommand{\st}{\textup{s.t.}}
\newcommand{\Tr}{\operatorname{Tr}}

\newif\ifmynotes
\mynotestrue
\usepackage{amsmath,amssymb,amsthm}

\usepackage{natbib}
\begin{document}

\title{A Probabilistic Model for Non-Contrastive Learning}

\author{Maximilian Fleissner, Pascal Esser, Debarghya Ghoshdastidar
% \IEEEmembership{Fellow, IEEE}, Second B. Author, and Third C. Author, Jr., \IEEEmembership{Member, IEEE}
% \thanks{This paragraph of the first footnote will contain the date on which you submitted your paper for review. It will also contain support information, including sponsor and financial support acknowledgment. For example, ``This work was supported in part by the U.S. Department of Commerce under Grant BS123456.'' }
\thanks{
This work has been supported by the German Research Foundation (Priority Program SPP 2298, project GH 257/2-1) and the DAAD programme Konrad Zuse Schools of Excellence in Artificial Intelligence, sponsored by the Federal Ministry of Education and Research.}
\thanks{All authors are with the Technical University of Munich, Germany (e-mail:  \{fleissner,esser,ghoshdas\}@cit.tum.de)}
% \thanks{M. Fleissner is with the Technical University of Munich, Germany (e-mail:  fleissner@cit.tum.de)}
% \thanks{P. Esser is with the Technical University of Munich, Germany (e-mail:  esser@cit.tum.de)}
% \thanks{D. Ghoshdastidar is with the Technical University of Munich, Germany (e-mail:  ghoshdas@cit.tum.de)}
% \thanks{The next few paragraphs should contain the authors' current affiliations, including current address and e-mail. For example, F. A. Author is with the National Institute of Standards and Technology, Boulder, CO 80305 USA (e-mail: author@boulder.nist.gov).}
% \thanks{S. B. Author, Jr., was with Rice University, Houston, TX 77005 USA. He is now with the Department of Physics, Colorado State University, Fort Collins, CO 80523 USA (e-mail: author@lamar.colostate.edu).}
}

\markboth{Journal of \LaTeX\ Class Files, Vol. 14, No. 8, August 2015}
{Shell \MakeLowercase{\textit{et al.}}: Bare Demo of IEEEtran.cls for IEEE Journals}
\maketitle

\begin{abstract}
  Self-supervised learning (SSL) aims to find meaningful representations from unlabeled data by encoding semantic similarities through data augmentations. Despite its current popularity, theoretical insights about SSL are still scarce. For example, it is not yet known how commonly used SSL loss functions can be related to a statistical model, much in the same as OLS, generalized linear models or PCA naturally emerge as maximum likelihood estimates of an underlying generative process. In this short paper, we consider a latent variable statistical model for SSL that exhibits an interesting property: Depending on the informativeness of the data augmentations, the MLE of the model either reduces to PCA, or approaches a simple non-contrastive loss. We analyze the model and also empirically illustrate our findings.
\end{abstract}

\begin{IEEEkeywords}
Self-supervised learning, Principal component analysis
% Enter key words or phrases in alphabetical order, separated by commas. For a list of suggested keywords, send a blank e-mail to keywords@ieee.org or visit \url{http://www.ieee.org/organizations/pubs/ani_prod/keywrd98.txt}
\end{IEEEkeywords}

\IEEEpeerreviewmaketitle

\section{Introduction}

The goal of self-supervised learning (SSL) is to learn meaningful representations of data by leveraging knowledge of semantic similarities within the data. As an example from the vision domain, consider two rotated versions $x,x^+$ of the same image. Intuitively, this rotation does not change the semantic content of the image. The goal of SSL is therefore to learn an embedding function $f$ that maps all pairs $x,x^+$ close together. 
SSL has its roots in the work of \citet{bromley1993signature} and recent deep learning based SSL shows great empirical success in computer vision, video data, natural language tasks and speech \citep{misra2020self,fernando2017self,chen2020simple,Steffen2019arxiv}. %Mohamed2022
While the goal of SSL is ultimately to learn a ``good representation'' $f$, this in itself is an ill-defined objective unless one explicitly takes a downstream task such as clustering or classification into consideration \citep{bengio2013representation}. However, downstream performance alone is an unsatisfactory measure from a statistical perspective. It tells us very little about the implicit assumptions SSL makes about the data it observes, may not allow us to incorporate uncertainty or prior knowledge, and does not illuminate what distinguishes SSL from traditional representation learning approaches. One way to deal with these issues is by specifying a \emph{generative model} from which data is sampled, and to maximize its likelihood. In the supervised as well as some unsupervised learning settings, such generative models are long since known and provide a probabilistic justification for popular objective functions. Let us give ab brief recap. 
% Notably, this is not yet the case for SSL, and motivates us to take a closer look at the matter. Before doing so, let us briefly revisit the supervised and unsupervised setting.

\textbf{Supervised learning.} 
Assume we are given $n$ data points $x_1,\dots,x_n\in\bbR^d$ and corresponding labels $y_1,\dots,y_n\in\bbR$. The generative model is defines by a linear transformation $w$ from data to labels under additive noise $\varepsilon$. 
Under this generative model optimizing the maximum likelihood estimation (MLE), $\max_w\bbP(y_1,x_1,\dots,y_n,x_n|w)$ is equivalent to optimizing the \emph{ordinary least squared loss}, which has  a closed form solution and is commonly used to solve regression problems \citep{Bishop}. Formally, under noise $\varepsilon_i\sim\calN(0,\sigma^2)$,
\begin{align*}
    \underset{\text{Probabilistic Data Model}}{y_i = w^\top x_i+\varepsilon_i}\quad \underset{\text{MLE}}{\longrightarrow} \quad 
    \underset{\text{Loss Function}}{\frac{1}{n}\sum(x^\top _i w-y_i)^2},
\end{align*}
connects the generative process to the OLS objective.

\textbf{Unsupervised learning.}
It is less straightforward to define an appropriate generative model for unsupervised learning. For \emph{PCA (Principal Component Analysis)}, \citet{tipping1999probabilistic} and \citet{lawrence2005probabilistic} consider a generative model that assumes an unobserved, $k$-dimensional latent variable $z \sim \calN(0,I_k)$, with data $x$ then being sampled according to
\begin{align*}
    x | z \sim \calN(Wx + \mu, \sigma^2 I)
\end{align*}
where $\mu \in \bbR^d$ is the mean and $\sigma^2>0$. They show that maximum likelihood estimation of $W$ (marginalized over the latent variables $z$) essentially leads to classic PCA. For other unsupervised learning approaches such as \emph{ICA (Independent Component Analysis)} or \emph{Projection Pursuit}, there is typically no unique way of defining a generative process.

\textbf{Self-Supervised learning.} We consider a \emph{non-contrastive learning} setup, where for every data-point $x_i$ a positive sample $x_i^+$ is created using a semantically meaningful augmentation. The goal of SSL is to learn an embedding function $f$ that maps all pairs $x_i,x_i^+$ close together. Several loss functions for non-contrastive learning exist \citep{zbontar2021barlow, bardes2021vicreg}. Let us focus on a simple non-contrastive loss, optimized over a linear function class. Given pairs $(x_i, x_i^+)_{i=1}^n$ of points, one minimizes
\begin{align}\label{eq:simple_loss}
    \begin{split}
        &\mathcal{L}(W) = \frac{1}{n} \sum_{i=1}^n \left\| f(x_i) - f(x_i^+)\right\|^2 \\
        &\text{where } f(x) = W^\top x \text{ and } W \text{ has orthonormal columns}
    \end{split}
\end{align}
Note that the orthonormality condition on $W$ prevents the model from learning a trivial embedding, a phenomenon known as dimension collapse. Despite the simplicity of this loss, even here no generative model is known to yield the same objective. This raises the following fundamental question:
\begin{quote}
    \emph{Are there natural statistical model whose MLE recovers commonly used SSL loss functions? What distinguishes such models from unsupervised learning?}
\end{quote}

In this paper, we show that there indeed exists a generative model whose MLE is the maximizer of the simple non-contrastive loss \eqref{eq:simple_loss}. Our model also captures the intuition that self supervision can only surpass PCA if its augmentations carry information about the true underlying signal.

\section{A probabilistic model for self-supervision}\label{sec:model_def}

Following established models for probabilistic PCA \citep{tipping1999probabilistic, lawrence2005probabilistic}, we assume data $x$ is generated as an unknown linear transformation $W$ of a latent Gaussian $z$, under additive noise given by a covariance matrix $A$. From there, we add self-supervision to the mix: The augmentation $x^+$ is assumed to be a Gaussian random variable \emph{conditioned on $x$}, whose mean is the underlying example $x$ itself. Formally,
\begin{align}\label{eq:model_def}
\begin{split}
    &z \sim \mathcal{N}(0,I) \\
    &x|z \sim \mathcal{N}(Wz,A) \\
    &x_+ | x \sim \mathcal{N}(x,B)
\end{split}
\end{align}
Here, $B$ is some covariance matrix and $W \in \bbR^{d \times k}$ is an unknown matrix with orthonormal columns. Albeit mathematically simple, the model specified above reflects the intuition that the positive example $x^+$ should be a noisy version of its anchor $x$. Furthermore, it captures the idea that it depends on the interplay between noise and augmentations to what extent SSL or PCA help in recovering the underlying signal $W$. Let us illustrate this with two examples.

\begin{enumerate}
    \item Suppose $A = \sigma^2 I_d$ and $B = \epsilon^2 I_d$ for some $\sigma^2, \epsilon^2>0$. In that case, augmentations add nothing but isotropic noise to the data. This reduces the MLE of $W$ to the PCA solution, and we refer to it as the \emph{isotropic noise model}.
    \item In contrast, suppose $B= I_d - WW^\top$. Then, the positive pairs $x^+$ will lie orthogonal to the subspace spanned by $W$, see Figure \ref{fig:illustration}. In other words, augmentations will leave all signal information intact, and only add noise orthogonal to it. We therefore refer to such $B$ as the \emph{orthogonal noise model}. Whenever $A$ perturbs the spectrum of $W W^\top$, the MLE is no longer given by the PCA solution.
\end{enumerate}

\begin{figure}[t]
    \centering
    \includegraphics[width=\linewidth]{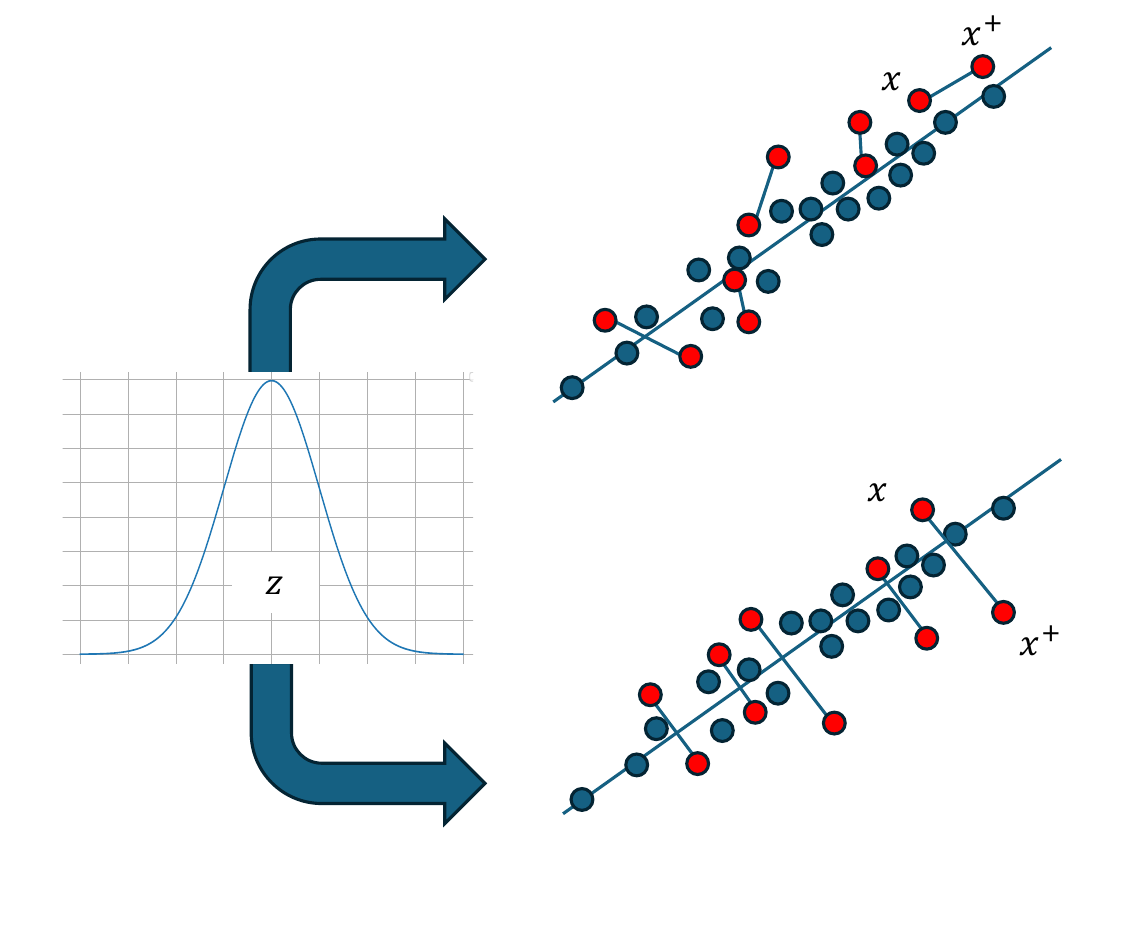}
    \caption{Model Illustration. Data $x$ is generated from a latent Gaussian, which is mapped linearly to a higher-dimensional space. Then, positive samples $x^+$ are generated as Gaussians conditioned on $x$. Depending on the covariance structure, we either recover an isotropic noise model (above) or the orthogonal noise model (below). In the latter, the positive pairs lie in a subspace orthogonal to the underlying signal direction.}
    \label{fig:illustration}
\end{figure}

\section{Maximum Likelihood Estimates}\label{sec:theory}

To derive the MLE, we express the log-likelihood in terms of $W$, $A$, $B$, and the data.

\begin{lemma}\label{lemma:likelihood}
    Suppose we draw $n$ positive pairs i.i.d. from the generative process \eqref{eq:model_def}, in the sense that $n$ independent ground truth samples $z_i$ are generated. Define the matrix
    \begin{align*}
        \Sigma =
        \begin{pmatrix}
        B & 0 \\
        0 & W W^\top + A
        \end{pmatrix} \in \bbR^{2d \times 2d}
    \end{align*}
    Then, the log-likelihood $L(W)$ of observing $\{(x_i,x_i^+)\}_{i=1}^n$ is maximized by minimizing
    \begin{align*}
        L(W) = \log \det (\Sigma) + \Tr \left( \Sigma^{-1} S \right)
    \end{align*}
    where we let
    \begin{align*}
        S = \sum_{i=1}^n 
        \begin{pmatrix} 
        &(x_i - x_i^+ ) (x_i - x_i^+)^\top & 0 \\ 
        &0  &x_i x_i^\top
        \end{pmatrix}
    \end{align*}
\end{lemma}

\begin{proof}
Write $\calN(u; \mu, s^2)$ for the density of a Gaussian $\calN(t,s^2)$ evaluated at $u$. For all $i \in [n]$, define $\Delta_i = x_i - x_i^+$, and observe that $\Delta_i | x_i \sim \calN(0, B)$. Moreover, $\Delta_i$ is independent of $x_i$. For all pairs $(\Delta_i, x_i)$, we obtain
\begin{align*}
    \begin{split}
        &\bbP(\Delta_i, x_i | W) \\
        &= \bbP(\Delta_i | x_i, W) \bbP(x_i|W) \\
        &= \bbP(\Delta_i | x_i, W) \int_{\bbR^k} \bbP(x_i | W, z_i) \; d \bbP(z_i) \\
        &= \calN(\Delta_i; 0, B) \int_{\bbR^k} \calN(x_i; W z_i, A) \cdot \calN(z_i; 0, I_k) \; dz_i \\
        &= \calN(\Delta_i; 0, B) \cdot \calN(x_i; 0, W W^\top + A)
    \end{split}
\end{align*}
Thus, we see that each pair $(\Delta_i, x_i)$ follows a $2d$-variate Gaussian with zero mean and block covariance
\begin{align*}
    \Sigma = 
    \begin{pmatrix}
    B & 0 \\
    0 & W W^\top + A
    \end{pmatrix}
\end{align*}
By independence of all $z_i$,
\begin{align*}
    \begin{split}
        &\log \bbP \left( \{ \Delta_i, x_i \}_{i=1}^n | W \right) = \sum_{i=1}^n \log \bbP( \Delta_i, x_i | W ) \\
        &= - \frac{1}{2} \sum_{i=1}^n k \log(2 \pi) + \log \det ( \Sigma) \\
        &\quad+ \Tr \left( B^{-1} \Delta_i \Delta_i^\top + \left( W W^\top + A \right)^{-1} x_i x_i^\top \right).
        \end{split}
\end{align*}
% with $c= k \log(2 \pi)$.
Maximizing this expression is equivalent to minimizing
\begin{align*}
    L(W) = \log \det (\Sigma) + \Tr \left( \Sigma^{-1} S \right)
\end{align*}
as claimed.
\end{proof}

When $W$ is orthonormal, it depends on $A,B$ whether the model results in PCA or the simple non-contrastive loss.

\begin{proposition}\label{prop:PCA}
    Consider the generative process with $A = \sigma^2 I_d$ and $B = \epsilon^2 I_d$. In that case, the MLE of $W$ is given by the top $k$ eigenvectors of the matrix $S_x = \sum_{i=1}^n x_i x_i^\top$.
\end{proposition}

\begin{proof}
    When $B$ is independent of $W$, the optimization problem from Lemma \ref{lemma:likelihood} reduces to minimizing
    \begin{align*}
        L(W) = \log \det(A) + \Tr( \Sigma^{-1} S)
    \end{align*}
    where we used the fact that $\log \det(\Sigma) = \log \det(A) + \log \det( \epsilon^2 I_d)$. Since $W$ has orthonormal columns, all eigenvalues of $W W^\top + A = W W^\top + \sigma^2 I_d$ are equal to $\sigma^2$ or $1+\sigma^2$, and we can also drop the log determinant from the minimization. The Woodbury formula and $W^\top W = I_k$ yield that
    \begin{align*} \begin{split}
        \Sigma^{-1} 
        &= 
        \begin{pmatrix} 
            \epsilon^{-2} I_d & 0 \\
            0 & \left( WW^\top + \sigma^2 I_d \right)^{-1} 
        \end{pmatrix} \\
        &=
        \begin{pmatrix} 
            \epsilon^{-2} I_d & 0 \\
            0 & \sigma^{-2} \left( I_d - \frac{1}{1+\sigma^2} W W^\top \right)
        \end{pmatrix}
    \end{split} \end{align*}
    Only the lower block depends on $W$, and so we must simply minimize
    \begin{align*}
        \sigma^{-2} \Tr \left( S_x \left( I_d - \frac{1}{1+\sigma^2} W W^\top \right) \right)
    \end{align*}
    where $S_x = \sum_{i=1}^n x_i x_i^\top $. For any fixed noise level $\sigma^2>0$, this is the same as minimizing
    \begin{align*}
       - \Tr \left( W^\top S_x W \right) 
    \end{align*}
    over $W^\top W = I_k$, which is achieved when $W$ corresponds to the top $k$ eigenvectors of the matrix $S_x$.
\end{proof}

This captures our intuition that, unless augmentations contain information on the signal $W$, SSL is no different from classic PCA. On the other hand, one can show the following.
    
\begin{proposition}\label{theo:ortho}
    Consider the generative process with $A = \rho I_d - W W^\top$ and $B = \gamma I_d - W W^\top$ for some $\rho, \gamma >1$. In that case, the MLE of $W$ is given by the top $k$ eigenvectors of the matrix $S_\Delta = -\sum_{i=1}^n (x_i - x_i^+) (x_i - x_i^+)^\top$, which is equivalent to minimizing the loss \eqref{eq:simple_loss}.
\end{proposition}

\begin{proof}
    As in the previous proof, the log determinant of $\Sigma$ stays constant because $W^\top W = I_k$, since this fixes all eigenvalues of $\gamma I_d - W W^\top$ and $\rho I_d - W W^\top$. Thus, we are left with a trace minimization, for which we compute
    \begin{align*} \begin{split}
        &\Sigma^{-1} 
        = 
        \begin{pmatrix} 
            \left( \gamma I_d - W W^\top \right)^{-1} & 0 \\
            0 & \left( \rho I_d - W W^\top + W W^\top  \right)^{-1} 
        \end{pmatrix} \\
        &=
        \begin{pmatrix} 
            \gamma^{-1} \left( I_d - \frac{1}{1-\gamma} W W^\top \right) & 0 \\
            0 & \rho^{-1} I_d
        \end{pmatrix}
    \end{split} \end{align*}
    Defining
    \begin{align*}
        S_\Delta = \sum_{i=1}^n \Delta_i \Delta_i^\top
    \end{align*}
    our minimization problem reduces to maximization of
    \begin{align*}
        \frac{1}{\gamma - \gamma^2} \Tr \left( W^\top S_\Delta W \right) 
    \end{align*}
    over $W^\top W = I_k$. Since $\gamma>1$, the optimal solution is given by the top $k$ eigenvectors of the matrix $-S_{\Delta}$. Equivalently, we can rewrite this as minimization of
    \begin{align*}
        &\sum_{i=1}^n \Tr \left( \left( f(x_i) - f(x_i^+) \right) (\left( f(x_i) - f(x_i^+) \right)^\top \right) \\
        &= \sum_{i=1}^n \left\| f(x_i) - f(x_i^+) \right\|^2
    \end{align*}
\end{proof}

\begin{remark}
    On image datasets, simply adding Gaussian noise often constitutes a useful augmentation. This is no contradiction to our theory: In practice, we would not use a linear function class, but instead a deep neural network or kernel model. In that case, we may write $f(x) = \langle w, \phi(x) \rangle$ for a nonlinearity $\phi$. Now, adding Gaussian noise to input images does \textbf{not} translate to Gaussian noise in the function space, due to the nonlinearity in $\phi$.
\end{remark}

The formulation opens up the possibility to use a simple non-contrastive loss in a Bayesian setting. By incorporating a prior on the matrix $W$, we obtain the posterior over $W$ given observations $(x_i, x_i^+)_{i=1}^n$ as
\begin{align*}
    \bbP(W | (x_i, x_i^+)_{i=1}^n) \propto \bbP(W) \cdot \bbP((x_i, x_i^+)_{i=1}^n | W) 
\end{align*}
Prior knowledge of the signal $W$ can in this way be introduced into the analysis. Moreover, our representations $f(x) = W^\top x$ are now probabilistic, and as such incorporate uncertainty. This uncertainty can then be ``passed on'' to any predictor $g$ learned on top of $f$, for example in downstream tasks.

\section{Numerical Simulation}

We illustrate the theoretical insights numerically and extend them beyond the setting of Model \eqref{eq:model_def} to Gaussian mixtures in the latent space.

\subsection{Theoretical Model}

We start our analysis by considering the model as outlined in \eqref{eq:model_def}, more specifically under (i) $A = \sigma^2 I_d\quad B = \epsilon^2 I_d$ and (ii) $A = \rho I_d - W W^\top,\quad B = \gamma I_d - W W^\top$. The goal is to characterize how close the embedding obtained by the SSL objective (Proposition~\ref{theo:ortho}) and the PCA objective (Proposition~\ref{prop:PCA}) is to the true latent structure. Therefore we define $\calL_{SSL}:=\norm{z-\hat{z}_{SSL}}$ where $\hat{z}_{SSL}$ is computed using $\hat{W}$ as defined in Proposition~\ref{theo:ortho}. Analogously $\calL_{PCA}$ is computed using the embedding following Proposition~\ref{prop:PCA}.

In Figure~\ref{fig:enter-label} we show the difference between both models, in both data settings. We can observe in Figure~\ref{fig:enter-label} (left) that for the orthogonal noise model, the SSL model outperforms PCA, for all $\rho, \gamma$. Interestingly, the difference is highest, when $\rho,\gamma$ are close to one and decreases with an increase in both parameters. Furthermore, Figure~\ref{fig:enter-label} (right) shows that for isotropic noise, the difference between both models is around zero, in line with the theoretical results stated in Proposition~\ref{prop:PCA} \& \ref{theo:ortho}.

\begin{figure}[t]
    \centering
    \includegraphics[width=\linewidth]{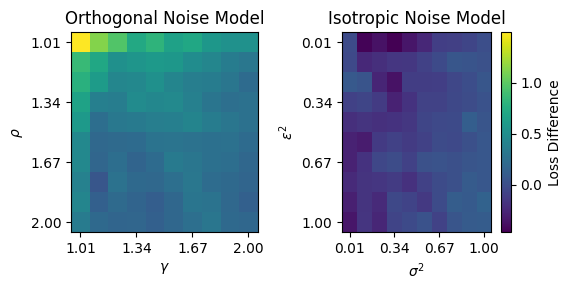}
    \caption{Numerical analysis of the theoretical setting.  Plotted is $\calL_{PCA}-\calL_{SSL}$ (therefore positive values imply the SSL model recovers the true embedding better then PCA), the averaged over $1000$ initializations.
    \textbf{(left)} orthogonal noise model for varying $\rho,\gamma\in(1.01,2)$. 
    \textbf{(right)} isotropic nose model for varying $\sigma^2,\epsilon^2\in(0.01,1)$.
    }
    \label{fig:enter-label}
\end{figure}

\begin{figure}[t]
    \centering
    \includegraphics[width=\linewidth]{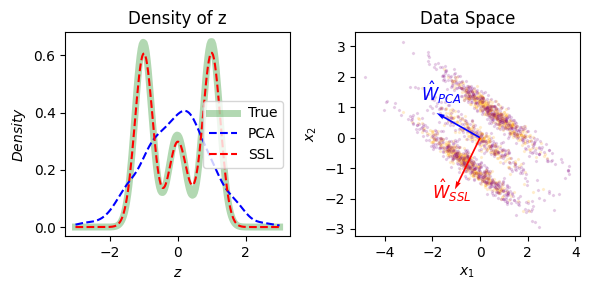}
    \caption{Orthogonal noise model with GMM latent space. \textbf{(left)} plotted is the density (under a Gaussian kernel estimate) of the true latent space in green and the representation by PCA in blue and SSL (according to Proposition~\ref{theo:ortho}) in red.
    \textbf{(right)} Plotted are the data (in purple) and positive samples (in orange) in $\bbR^2$.
    }
    \label{fig:GMM}
\end{figure}

\subsection{Gaussian Mixture Model}

Our theoretical analysis is built on the assumption of Gaussian latent variables. We now extend the model presented in \eqref{eq:model_def} by considering a more complex latent structure, namely a mixture of $K$ Gaussian. Define:
\begin{align*}
\begin{split}
    &z\sim \sum_i^K\pi_i\calN(\mu_i,\Sigma_i) \quad\st\quad\sum_i^K\pi_i =1 \\
    &x|z \sim \mathcal{N}(Wz,A) \\
    &x_+ | x \sim \mathcal{N}(x,B)
\end{split}
\end{align*}
where $\mu_i,\Sigma_i$ are the mean and covariance of component $i$. Figure~\ref{fig:GMM} illustrates this for three mixture components\footnote{Specifically for the data we consider the following hyper-parameters: $A = \rho I_d - W W^\top,\quad B = \gamma I_d - W W^\top,\quad\rho, \gamma >1$ as well as $k=1,d=2,n=1000,\pi_1=\pi_2=0.4,\pi_3=0.2,\rho= \gamma=1.01$}. Figure~\ref{fig:GMM} (right) shows the data and augmentations, as well as the weights obtained by PCA and SSL.
In addition, in Figure~\ref{fig:GMM} (left), we observe that the embedding obtained by SSL aligns with the true latent distribution, whereas PCA does not learn the direction of modes. This example highlights that PCA just projects the data along the direction of maximum variance, which is orthogonal to the true signal direction. \textit{In comparison, SSL manages to take advantage of the signal direction encoded in the positive sample to project the data into the correct direction.}

\section{Related Work}

Before concluding, we discuss some additional related works regarding the theoretical analysis of SSL and generative models for it.

\emph{Theoretical focus of SSL analysis.} While there are theoretical works addressing  the importance of data augmentation \citep{wen2021toward, Zhuo0M023} in SSL, and  characterizing causes for its ``dimension collapse'' \citep{pokle2022contrasting,esser2023representation}, the main focus of the theoretical literature on SSL has been on providing generalization error bounds for downstream tasks\citep{Arora2019ATA,WeiXM21}.

\emph{Analysis of SSL models under generative data settings.} Several works that consider a ``Bayesian SSL'' setup mainly focus on assuming a prior in the learned models, not a prior on the data model \citep{liu2023bayesian,vahidi2023probabilistic}.
Closest to our work is \citep{bizeul2024probabilistic} that considers generative functions that impose cluster structures. By computing Evidence Lower Bounds they connect it to existing loss functions. \citet{zimmermann2021contrastive} assume data is generated by an unknown injective generative model that maps latent variables from a hypersphere to observations on a manifold. This model is then connected to optimizing InfoNCE \citep{oord2018representation}.
Importantly, while in both approaches a generative model is assumed, there is no explicit connection to the creation of positive samples, or what constitutes a meaningful augmentation, which is the focus of this work. Therefore, \citep{zimmermann2021contrastive,oord2018representation} does not allow for a direct analysis of data dependent cases in which SSL is beneficial over other standard models (such as PCA).

\section{Discussion and Open Questions}

While the model we consider is simple, it formalizes the idea that SSL does not provide additional insights over PCA unless the augmentation carries specific signal information that would otherwise vanish due to uninformative variance.
Of course, assuming that the positive pairs $(x,x^+)$ lie exactly along a subspace orthogonal to the signal direction is a strong assumption that cannot be expected to be true in practice. However, it demonstrates that the ideal function space within which the simple non-contrastive loss \eqref{eq:simple_loss} should be minimized is the one in which this condition (approximately) holds. This immediately opens up two questions.

\textbf{Practical used augmentations.} For commonly used augmentations (cropping, rotation, blurring) on image classification tasks, which feature transformation $\phi$ is the one that leads to an orthogonal noise model justifying the effectiveness of SSL? Is it true for popular machine learning kernels? 

\textbf{Influence of optimization.} Is there something in the implicit bias of gradient-based deep learning that leads to such ``good'' feature transformations $\phi$ being learned under commonly used SSL loss functions?

%\clearpage
\bibliography{main}
\clearpage

\end{document}